\documentclass[11pt]{article}
\usepackage[latin9]{inputenc}
\usepackage{geometry}
\usepackage{amsmath}
\usepackage{amssymb}
\usepackage{amsthm}
\usepackage{algorithm}
\usepackage{algorithmic}
\usepackage{float}
\usepackage{graphicx}
\usepackage{mathtools}
\usepackage{pgffor}
\usepackage[numbers, sort]{natbib}
\usepackage{fancyhdr}
\usepackage[title,titletoc]{appendix}

\graphicspath{{figures/}}
\makeatletter

\usepackage{amsfonts}\usepackage{nopageno}
\reversemarginpar
\newtheorem{theorem}{Theorem}

\newtheorem{lemma}{Lemma}
\newcommand{\trace}{{\rm trace}}
\newcommand{\be}{\begin{equation}}
\newcommand{\ee}{\end{equation}}
\newcommand{\mR}{{\mathbb R}}

\makeatother

\begin{document}
\title{Optimal transport for vector Gaussian mixture models}
\author{Jiening Zhu, Kaiming Xu, Allen Tannenbaum
\thanks{J.\ Zhu and K.\ Xu are with the Department of Applied Mathematics \& Statistics, Stony Brook University, NY; email: jiening.zhu@stonybrook.edu, kaiming.xu@stonybrook.edu}
\thanks{A.\ Tannenbaum is with the Departments of Computer Science and Applied Mathematics \& Statistics, Stony Brook University, NY; email: allen.tannenbaum@stonybrook.edu}}
\date{\today}
\maketitle
\thispagestyle{fancy}
      \lhead{}
      \chead{}
      \rhead{}
      \lfoot{}
      \rfoot{}
      \cfoot{\thepage}
      \renewcommand{\headrulewidth}{0pt}
      \renewcommand{\footrulewidth}{0pt}
\pagestyle{fancy}
\cfoot{\thepage}
\begin{abstract}
Vector-valued Gaussian mixtures form an important special subset of vector-valued distributions. In general, vector-valued distributions constitute natural representations for physical entities, which can mutate or transit among alternative manifestations distributed in a given space. A key example is color imagery. In this note, we vectorize the Gaussian mixture model and study several different optimal mass transport related problems associated to such models. The benefits of using vector Gaussian mixture for optimal mass transport include computational efficiency and the ability to preserve structure.

\textit{This paper is dedicated to our dear friend and colleague, Professor Peter Olver, on the occasion of his 70th birthday. Happy Birthday, Peter!}
\end{abstract}

\section{Introduction}
Finite mixture models can describe a wide range of statistical phenomena. They have been successfully applied to numerous fields including biology, economics, engineering, and the social sciences \cite{McLachlan2000}. The first major use and analysis of mixture models is perhaps due to the mathematician and biostatistician Karl Pearson over 120 years ago, who explicitly decomposed a distribution into two normal distributions for the characterization of the non-normal attributes of forehead to body length ratios in female shore crab populations \cite{October1894}. The literature on analyzing and applying mixture models is growing due to their simplicity, versatility and flexibility. One of the most commonly used mixture models is the Gaussian mixture model (GMM), which is a weighted sum of Gaussian distributions.

Optimal mass transport (OMT) has been a major subject of mathematical research, originating with the French civil engineer and mathematician Gaspard Monge in 1781 \cite{villani1, villani2}. OMT allows one to define a distance between two probability distributions, which makes it a very powerful tool to analyze the geometry of distributions. Its applications include but not limited to signal processing, machine learning, computer vision, meteorology, statistical physics, quantum mechanics, and network theory \cite{Arjovsky2017,Haker2004,rr,Statement2020}. Milestones of this subject include the seminal work of Leonid Kantorovich \cite{villani1, villani2}, who relaxed the original problem so that it can be solved through linear programming, and Benamou and Brenier \cite{BB} who introduced a computational fluid dynamics (CFD) approach to OMT. More recent developments involve extensions of the theory to the vector-valued, matrix-valued and unbalanced cases \cite{vectorvalued,CheGeoTan16,CheGeoTan16b,Chizat}.

The problem that motivated the present work arose when the authors were working with certain medical image data. The object was to compute optimal mass transport while preserving key structures. The authors of \cite{Chen2019} studied OMT for GMM, which however can only work on single layered data, e.g., gray scale images. The need for working directly on the original color images with the potential of capturing more information inspired us to generalize the OMT setting from the one-layered case to the three-layered case. More generally, in this note, we develop optimal transport for vector-valued Gaussian mixture models, which can have any dimension and any general connection structures among the layers. Furthermore, corresponding to unbalanced OMT, we also develop an unbalanced version for Gaussian mixture models.

There have several relevant works in the literature describing various versions of OMT to GMMs and vector-valued data as well as extending the theory to manifolds, which we would like to review here in order to put the present work in proper perspective.
First of all, Fitschen, Laus, and Schmitzer \cite{schmitzer_manifold} develop a rigorous transport theory for manifold-valued images. Delon and Desolneux \cite{delon_gmm} study a version of OMT for GMMs (with some beautiful examples), essentially equivalent to the work proposed in \cite{Chen2019} and followed in the present work. Fitschen, Laus and Steidl \cite{fitschen_rgb} formulate a dynamical model of transport for discrete RGB color images inspired by the work Benamou-Brenier \cite{BB}. In the work of Thorpe \emph{et al.} \cite{thorpe}, a transport-based distance is defined and studied, which is directly applicable to general, non-positive and multi-channel signals.


In what follows, we will first give some background on GMM and OMT. Next, we summarize some of the work of \cite{Chen2019}, and then introduce two different approaches for the vector-valued case. We investigate the unbalanced GMM problem and conclude with some illustrative numerical results.



\section{Gaussian mixture models}
A Gaussian mixture model is one of the most important examples of a mixture model. Mathematically, a GMM is a probability distribution which is the weighted sum of several Gaussian distributions in $\mR^N$. Namely, an \emph{$n$-component Gaussian mixture model (GMM)} is given by
\begin{equation}\label{gmm}
  \mu=p_1\nu_1+p_2\nu_2+\cdots+p_n\nu_n.
\end{equation}
Here
\begin{equation}\label{gaussian function}
  \nu_i(x)=\frac{1}{\sqrt{(2\pi)^{N}|\Sigma_i|}}\exp\{-\frac{1}{2}(x-m_i)^T\Sigma_i^{-1}(x-m_i)\},
\end{equation}
where $m_i\in\mR^N$ is the mean and $\Sigma_i\in\mR^{N\times N}$ is the positive definite covariance matrix for $1\le i\le n$. Further,
\begin{equation}\label{probability1}
   \sum_{i=1}^{n}p_i=1,\quad p_i>0, \forall i\in \{1,...,n\}
\end{equation}
so that $\mu$ is a probability distribution.

We denote the set of all the GMMs in $\mR^N$ by $\mathcal{G}(\mR^N)$. It is a dense subset of the set of all the probability distributions in the sense of the weak$*$ topology \cite{Stergiopoulos_2000}. Thus one can use GMM to fit a distribution with arbitrarily small error. Of course, this may involve a very large number of Gaussians.

\section{Optimal mass transport}
In this section we sketch the basics of optimal mass transport. See \cite{villani1, villani2} for all the details as well as an extensive list of references. In the present work,  we only consider absolutely continuous measures, which thus have  density functions representations. By slight abuse of notation and terminology, we will identify the given measure with its density function representation.

The original formulation of OMT due to Gaspard Monge may be expressed as follows:
\begin{equation}
\label{1}
  \inf_T\{\int_{E}c(x,T(x))\rho_0(x)dx\ |\ T_{\#}\rho_0=\rho_1\},
\end{equation}
where $c(x,y)$ is the cost of moving unit mass from $x$ to $y$, which is a lower semi-continuous and bounded below, $T$ is the transport map, and $\rho_0,\rho_1$ are two probability distributions defined on $E$, a subdomain of $\mathbb{R}^n$. $T_\#$ denotes the push-forward of $T$ of corresponding measures of the distributions.



As pioneered by Leonid Kantorovich, the Monge formulation of OMT may be relaxed replacing transport maps $T$ by couplings $\pi$:
\begin{equation}
  \inf_{\pi\in\Pi(\rho_0,\rho_1)}\int_{E\times E}c(x,y)\pi(dx,dy),
\end{equation}
where $\Pi(\rho_0,\rho_1)$ denotes the set of all the couplings between $\rho_0$ and $\rho_1$ (joint distributions whose marginal distributions are $\rho_0$ and $\rho_1$).

The discrete Kantorovich form may be written as follows:
\begin{equation}
  \min_{\pi\in\Pi(\rho_0,\rho_1)}\sum_{i}\sum_{j}c(i,j)\pi(i,j),
\end{equation}
where $\rho_0\in\mathbb{R}_{+}^m,\rho_1\in\mathbb{R}_{+}^n$ are two discrete probability density functions ($\sum_i^m\rho_0(i)=\sum_j^n\rho_1(j)=1$), $\Pi(\rho_0,\rho_1)$ is the set of matrices $\{\pi\in\mR_+^{m\times n}|\pi \vec{1}_n=\rho_0,\ \pi^T \vec{1}_m=\rho_1\}$, and $\vec{1}_m$ and $\vec{1}_n$ are vectors all 1's of length $m$ and $n$, respectively. $c(\cdot,\cdot)$ is a discrete cost function. Kantorovich form is guaranteed to have a optimal solution ($\rho_0\otimes\rho_1^T\in\Pi(\rho_0,\rho_1)$) while in some cases Monge form might admit no feasible solution.  

One may show that for $c(x,y) =\|x-y\|^2$ (square of distance function), the Kantorovich and Monge formulations are equivalent in the absolutely continuous measure case; see \cite{villani1, villani2} and the references therein. Moreover for $c(x,y)=||x-y||^2$, the specific infimum is called \emph{Wasserstein-2 distance} ($\mathcal{W}_2$).

\section{Optimal mass transport for Gaussian mixture models}
We are interested in looking at optimal interpolation paths from GMM to another, that is geodesic paths in the space of probability distributions \cite{Otto}. The problem is that for general GMMs with more than one summands, the optimal path goes out of the subspace of GMMs, that is, the GMM structure is lost. This was exactly the motivation underlying the work of \cite{Chen2019}. There are several advantages of preserving the GMM structure including greatly saving computational cost via dimension reduction.

\subsection{OMT between Gaussian distributions}
For two Gaussian distributions $\mu_i,\ i=0,1$ whose means and covariances are $m_i$ and $\Sigma_i$, respectively, it is well-known \cite{villani1,villani2} that the $\mathcal{W}_2$ distance between $\mu_0$ and $\mu_1$ has a closed form solution:
\begin{equation}\label{W2gaussian}
  \mathcal{W}_2(\mu_0,\mu_1)^2=||m_0-m_1||^2+\trace(\Sigma_0+\Sigma_1-2(\Sigma_0^{1/2}\Sigma_1\Sigma_0^{1/2})^{1/2}).
\end{equation}
For each $t\in[0,1]$, the distribution $\mu_t$ on the geodesic path is a Gaussian whose mean and covariance matrix are defined as follows:
\begin{align}\label{gaussian_geodesic}
  &m_t=(1-t)m_0+tm_1 \\
  &\Sigma_t=\Sigma_0^{-1/2}((1-t)\Sigma_0+t(\Sigma_0^{1/2}\Sigma_1\Sigma_0^{1/2})^{1/2})^2\Sigma_0^{-1/2}.
\end{align}

\subsection{OMT between GMMs}
Let $\mu_0,\mu_1$ be two Gaussian mixture models of the form
\begin{equation*}
  \mu_i=p_i^1\nu_i^1+p_i^2\nu_i^2+\cdots+p_i^{n_i}\nu_i^{n_i},\ i=0,1.
\end{equation*}
Following \cite{Chen2019,delon_gmm}, the distance between $\mu_0,\mu_1$ is defined as
\begin{equation}\label{gmm kantorovich}
  d(\mu_0,\mu_1)^2=\min_{\pi\in\Pi(p_0,p_1)}\sum_{i,j}c(i,j)\pi(i,j),
\end{equation}
where
\begin{equation}\label{gmm_c}
  c(i,j)=\mathcal{W}_2(\nu_0^i,\nu_1^j)^2.
\end{equation}
As $\nu_0^i$ and $\nu_1^j$ are Gaussian distributions, the $\mathcal{W}_2$ distance may be computed as in (\ref{W2gaussian}). In \cite{Chen2019,delon_gmm}, it is proven that $d(\cdot,\cdot)$ is indeed a metric on $\mathcal{G}(\mR^N)$.
Further, the geodesic on $\mathcal{G}(\mR^N)$ connecting $\mu_0$ and $\mu_1$ is given by
\begin{equation}\label{mu_t}
  \mu_t=\sum_{i,j}\pi^*(i,j)\nu_t^{ij},
\end{equation}
where $\nu_t^{ij}$ is the displacement interpolation in (\ref{gaussian_geodesic}) between $\nu_0^i$ and $\nu_1^j$. $\pi^*(\cdot,\cdot)$ is the optimal solution of (\ref{gmm kantorovich}).


\section{Vector-valued GMM}
In this section, we extend the definition of GMM to the vector-valued case, based on which we will formulate generalizations of the work of \cite{Chen2019}.

\subsection{Vector-valued distributions}
A \emph{vector-valued distribution} has a corresponding  density function which is vector-valued. Formally, a \emph{vector-valued distribution},  $\rho=[\rho_1,...,\rho_M]$ on $\mR^N$, is a map from $\mR^N$ to $\mR^M_+$ such that
	\[
		\sum_{i=1}^M\int_{\mR^N} \rho_i(x)dx=1, 
	\]
with the connections among its $M$ channels, defined by a connected graph $G=(V,E),$ which has $M$ nodes and whose edges determine the connections. Thus, $\rho$ may be considered as a general distribution on $\mR^N\times G$, where $V=\{1,2,\cdots,M\}$ with $E$ defining the connections among the channels (layers). As described in \cite{vectorvalued}, it may represent a physical entity that may mutate or be transported among several alternative manifestations with certain relationships among its $M$ channels.

The Euclidean structure of $\mR^N$ and graph structure of $E$ together give a complete metric structure for $\mR^N\times G$, 
\begin{equation*}
    d^p((x,u),(y,w))=||x-y||^p+\gamma d_G^p(u,w),
\end{equation*}
where $(x,u),(y,w)\in \mR^N\times G$, are two points in the space, $p>0$, $||\cdot||$ is the norm of $\mR^N$ and $d_G(\cdot,\cdot)$ is the graph distance which is defined as the length of shortest path on $G$. The vector-valued OMT problem deals with transport on such a metric space. 

\begin{figure}[H]
  \centering
  \includegraphics[width=0.7\linewidth]{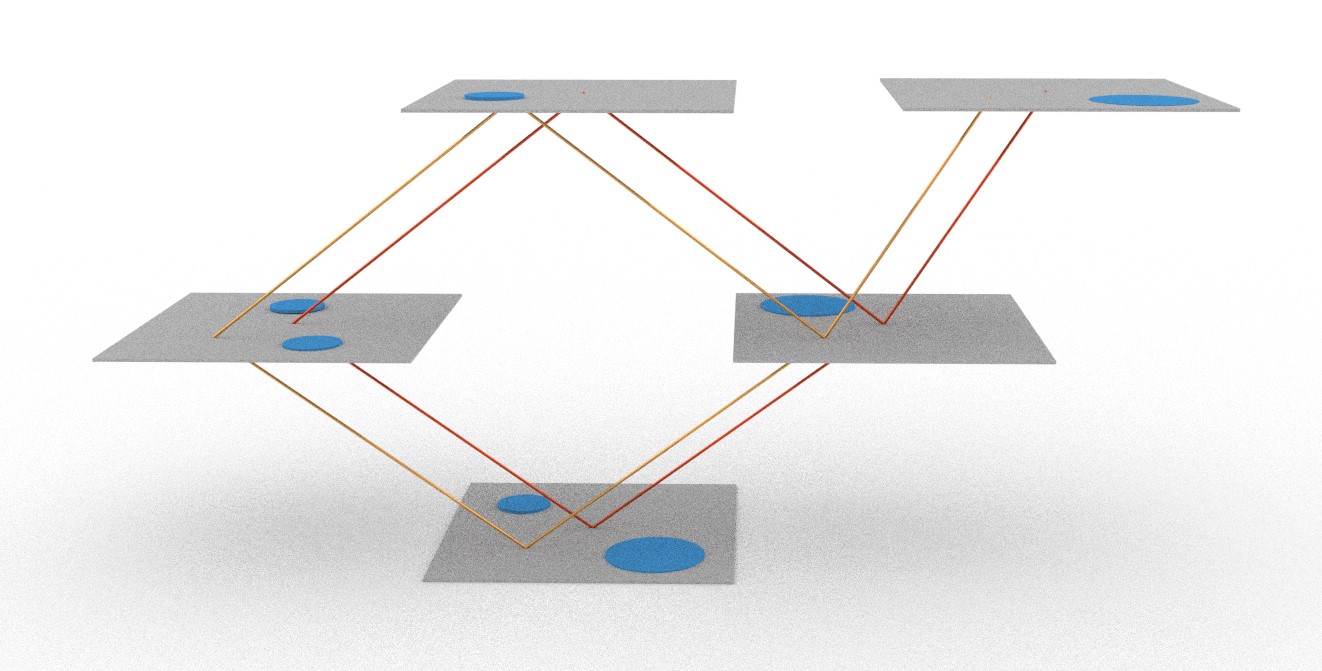}
  \caption{An example of a vector-valued distribution in $\mR^2$ with 5 channels and a specific graph structure. The distribution takes values only within each channel and its total sum is 1.}\label{vector-valued distribution}
\end{figure}

\subsection{Vector GMMs as a subset of vector-valued distributions}
{\em Vector-valued GMMs} are those vector-valued distributions such that the distribution in each layer is a weighted sum of Gaussians and the weights of the Gaussians sum up to 1. Formally,
\begin{equation}\label{vector gmm}
  \rho=p^1\nu^1\vec{\delta}_{q^1}+p^2\nu^2\vec{\delta}_{q^2}+\cdots+p^{n}\nu^{n}\vec{\delta}_{q^n},
\end{equation}
where $\vec{\delta}_k$ is a column vector which is the $k^{\rm th}$ column of the $M$ by $M$ identity matrix and $q^i$ is the index of channel where the $i^{\rm th}$ Gaussian lies in. We will always assume that the latter is a probability distribution, i.e.,
\begin{equation}
  \sum_{i=1}^{n}p^i=1.
\end{equation}

\section{Generalization of the OMT GMM framework  to vector-valued GMMs}\label{omttovecgmm}
Consider two vector-valued GMMs $\rho_0$ and $\rho_1$:
\begin{align*}
  \rho_i=&p_i^1\nu_i^1\vec{\delta}_{q_i^1}+p_i^2\nu_i^2\vec{\delta}_{q_i^2}+\cdots+p_i^{n_i}\nu_i^{n_i}\vec{\delta}_{q_i^{n_i}},\ i=0,1.
\end{align*}
We want to compute an OMT based distance and a displacement interpolation between these two vector-valued distributions with the requirement that the vector GMM structure is preserved along the interpolation path. In short, we want to construct the analogous framework of \cite{Chen2019}, but replace scalar-valued GMMs with vector-valued GMMs.

As above, let $\Pi(p_0,p_1)$ denote the set of joint probabilities with given marginals $p_0$ and $p_1$. Given a graph structure, the most straightforward approach is to require only certain parts of $\Pi$ to be nonzero, namely only when the the source and target Gaussians are in the same channel or when they are located in adjacent channels. A more detailed description is given in Appendix~\ref{modify_PI}.

\textbf{\emph{Unfortunately, this natural (and perhaps most straightforward) generalization may not admit a solution.}} Indeed, the newly added constraints on $\Pi$ may not work for general graph structures. Thus, the only other choice left in (\ref{gmm kantorovich}) is to modify the cost matrix $c(\cdot,\cdot)$.

Accordingly, we now propose two different approaches. Both of them modify the cost matrix $c(\cdot,\cdot),$ but from very different points of view.

\section{Approach 1: modify cost matrix}\label{another approach}
The first approach is based on the following intuition. Since we know how to deal with OMT on $\mR^n$, and noting that it is also easy to compute the shortest path on a connected graph using, e.g., the Bellman-Ford \cite{10.2307/43634538,Ford1956NETWORKFT} or Dijkstra algorithm \cite{Dijkstra1959}, our idea is to consider $\mR^n$ and $G$ separately, and combine them in the Kantorovich step (\ref{gmm kantorovich}) as a new cost term.

More precisely, as the one-channel GMM case, utilizing the Kantorovich step (\ref{gmm kantorovich}), we take
\begin{equation}\label{c1}
    c_1(i,j)=\mathcal{W}_2(\nu_0^i,\nu_1^j)+\gamma\tilde{d}_G(q_0^i,q_1^j),
\end{equation}
where $\tilde{d}_G(\cdot,\cdot)$ is the length of the shortest path on the graph $G$ and $\gamma$ is a weight parameter to control how much we allow the cross-channel transport.

With our new $c_1(i,j)$ defined, we can compute the distance as
\begin{equation}\label{vector gmm kantorovich}
  d_{V_1}(\rho_0,\rho_1)=\min_{\pi\in\Pi(p_0,p_1)}\sum_{i,j}c_1(i,j)\pi(i,j).
\end{equation}
As in the one-channel case, we can prove $d_{V_1}(\cdot,\cdot)$ is indeed a metric on $\mathcal{G}(\mR^N\times G)$.
\begin{theorem}\label{metric}
$d_{V_1}(\cdot,\cdot)$ defines a metric on $\mathcal{G}(\mR^N\times G)$
\end{theorem}
\begin{proof}
See Appendix~\ref{new distance}.
\end{proof}

The \emph{displacement interpolation} may be defined as follows:
\begin{equation}\label{vector geodesic}
  \rho_t=\sum_{i,j}\tilde{\pi}_1^*(i,j)\nu_t^{ij}\vec{\delta}_{path_G(q_0^i,q_1^j,t)},
\end{equation}
where $\tilde{\pi}_1^*(\cdot,\cdot)$ denotes the optimal solution of (\ref{vector gmm kantorovich}), $\nu_t^{ij}$ is the displacement interpolation in (\ref{gaussian_geodesic}) between $\nu_0^i$, and $\nu_1^j$, $path_G(\cdot,\cdot,t)$ is the path interpolation between two nodes of $G$. The value of $\vec{\delta}$ between two nodes is taken as the weighted sum of the two $\vec{\delta}$ vectors of the two nodes.

As a concrete example, suppose that the shortest path between the two nodes $node_1$ and $node_2$ is given by $node_1\rightarrow node_3\rightarrow node_4\rightarrow node_2$. Then at $t=0.5$ the interpolation point lies exactly in the middle of points 3 and 4, which is to say $path_G(node_1,node_2,0.5)=0.5*node_3+0.5*node_4$. Further, $\vec{\delta}_{path_G(node_1,node_2,0.5)}=\vec{\delta}_{0.5*node_3+0.5*node_4}=0.5*\vec{\delta}_3+0.5*\vec{\delta}_4$.

\begin{theorem}\label{geo}
The displacement interpolation (\ref{vector geodesic}) is a geodesic in the sense of $d_{V_1}(\cdot,\cdot)$
\end{theorem}
\begin{proof}
See Appendix~\ref{geodesic_proof}.
\end{proof}

\section{Approach 2: Continuous version of $\mathcal{G}(\mR^N\times G)$}
For vector-valued GMMs, one cannot directly apply the same OMT framework as in the scalar case \cite{Chen2019} since the last index is taken discretely. Thus we will generalize the framework by making the last index continuous as well. The basic idea is to consider a continuous problem and view the vector-valued distribution as a projection of the continuous solution onto the original discrete space.

More precisely, we propose to extend each point on the edges of the given graph $G$ instead of only taking values on vertices of the graph. Moreover, we extend each edge half-way from both ends, so that newly added points are centered at the original vertices of the graph. Thus, we consider the following point set of a continuous version of the graph $G$:
\begin{equation}
    G^{c}=\{u+t (w-u)|u,w\in V(G),u\sim w, t\in[-0.5,0.5]\}.
\end{equation}
Here $u,w\in V(G)$ are taken as abstract vertices, not as integers. In addition, we assign a length to each edge, $\gamma$, so that we are able to perform integration on that set. (We may consider the use of nonuniform edge lengths, in case we are given specific edge weights.)

So the continuous structure we are going to employ is $\mR^N\times G^c$. Though it is possible to realize this whole structure in a Euclidean space $\mR^{N+K}$, one finds an unnecessarily complicated topology space in general, as the two examples shown in Figures~\ref{fig.rgb}, \ref{fig.Klein} indicate.

\begin{figure}[H]
    \centering
    \includegraphics[width=0.7\linewidth]{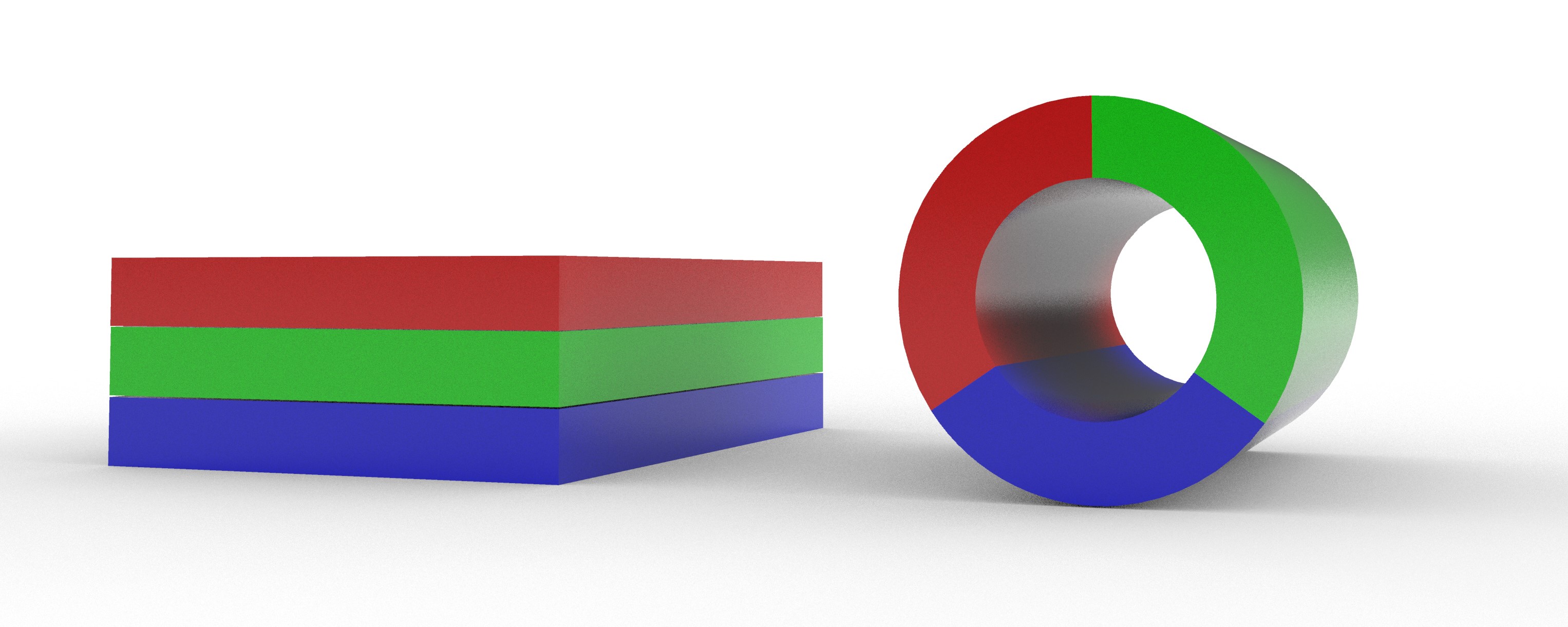}
    \caption{Consider a 3-vector distribution in $\mR^2$ (may be used for color image data). There are essentially two ways to connect the R,G and B channels. If the graph is given by $R\rightarrow G\rightarrow B,$ then its continuous version becomes a cube. On the other hand if additionally, R is also connected to B, then the continuous version is essentially a hollow cylinder, which is already a rather complicated space to employ in an OMT framework.}\label{fig.rgb}
\end{figure}

\begin{figure}[H]
    \centering
    \includegraphics[width=0.7\linewidth]{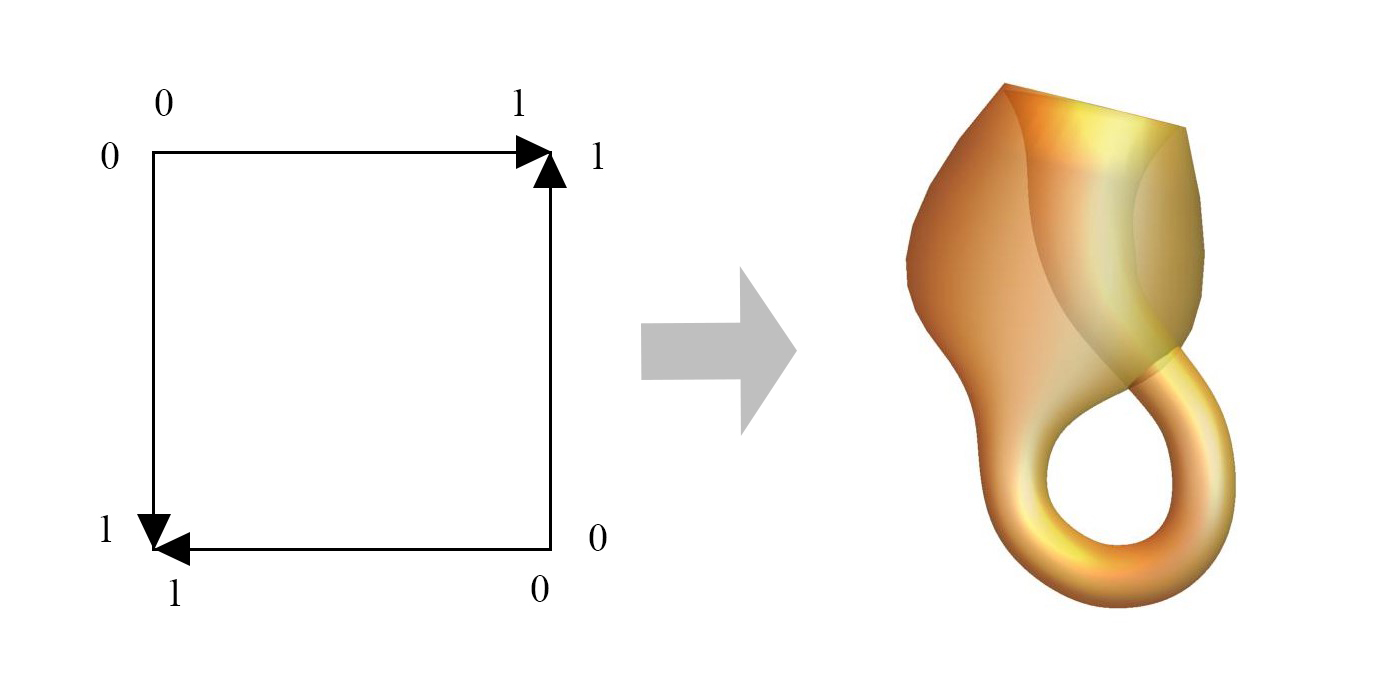}
    \caption{Consider a 4-vector distribution in $\mR$ given by four 1D edges ($[0,1]$) with each edge connected to the other three edges making a Klein bottle \cite{massey} with a hole on its surface, illustrated by the shape shown on the right. Computing optimal transport on its surface in general would be quite challenging.}\label{fig.Klein}
\end{figure}

In fact, we do not need to realize the global structure of the complicated space $\mR^N\times G^c$ as a whole. Instead, we can just consider the local structure. A natural and simple way to do that is to impose a manifold structure, which we will now elucidate. We denote the manifold by $\mathcal{M}$.

In order to define $\mathcal{M}$, we need to specify its atlas:
\begin{equation}
    A=\{\mR^N\times p| p\in [G^c]_0 \},
\end{equation}
where $[G^c]_0$ is a subset of all continuous paths on $G^c$ which have no cycles (no recurring vertices of $G$ on the path). We can characterize the charts as we stack layers (like ``bricks''), where we follow the order of the path on $G^c$. It is clear that each $p$ is homeomorphic to $\mR$, so that each chart is homeomorphic to $\mR^{N+1}$.

We want to define a distribution on $\mathcal{M}$ in such a manner that the original distribution is the projection of each layer's range. The projection is defined as the integral of the last index:
\begin{equation}
    P_u(f(x,z))=\sum_{w\sim u}\int_{-0.5}^{0.5}f(x,u+t(w-u))dt,
\end{equation}
where $P_u(\cdot)$ is the projection of the range of layer $u$, and $f$ is a distribution on the manifold. The integral range of the last index is the intersection of a ball centered at $u$ which has half-edge radius with $G^c$  (layer $u$'s range). Note that for different $w$'s which are connected to $u$, the ranges are like different orbits centered at $u$.

One of the simplest choices for lifting the original distribution to the manifold is a ``Gaussian cylinder,'' i.e., a product of a Gaussian distribution and a uniform distribution within the range of the layer. Thus, we accordingly thicken each Gaussian.

\begin{figure}[H]
    \centering
    \includegraphics[width=0.8\linewidth]{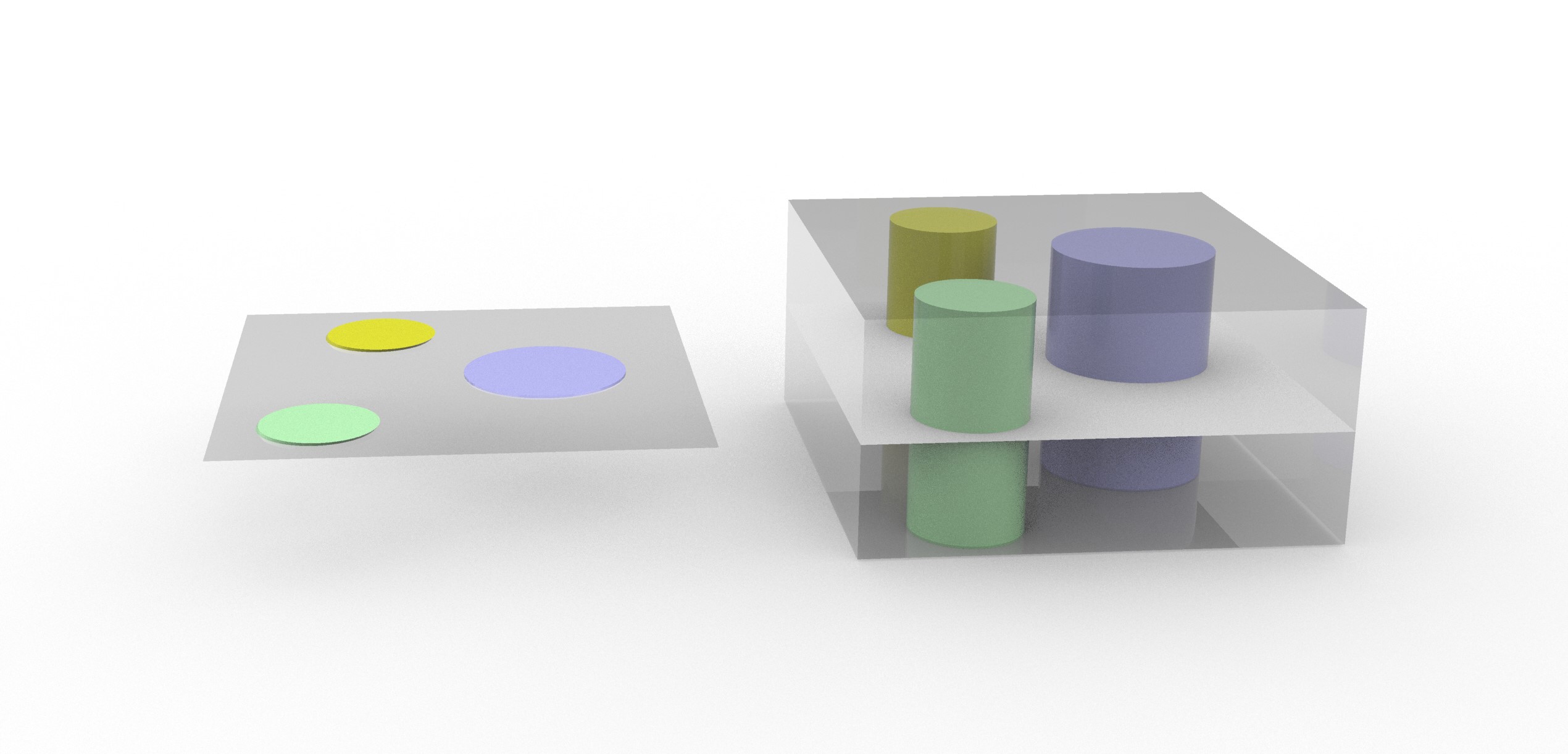}
    \caption{Left hand side is one of the layers of vector GMM. Right hand side is the chart centered at that layer. Gaussians in the original layer become "Gaussian cylinders" on the manifold.}
    \label{fig:onelayerchart}
\end{figure}

If a layer has more than one edge connected to it, then the original distribution may be lifted to multiple ``Gaussian cylinders'' (located at all the possible orbits that are centered at the given layer) with combined weights. Notice that even though the ``Gaussian cylinders'' project to be the same vector-valued distribution within the given layer, they may have different potentials to transport to different directions on the graph.


Let us briefly summarize the optimal transport problem we are going to solve on the manifold $\mathcal{M}$ with the approach we just introduced. Given the projection of each layer's range for the source (starting) and target (terminal) distributions, we want to find corresponding source and target distributions on $\mathcal{M}$ such that the transport cost is optimally low.
As before, we first consider the sub-problem where the starting and terminal vector-valued distributions are two Gaussian distributions, which may be located on different layers.

\begin{figure}[H]
    \centering
    \includegraphics[width=0.8\linewidth]{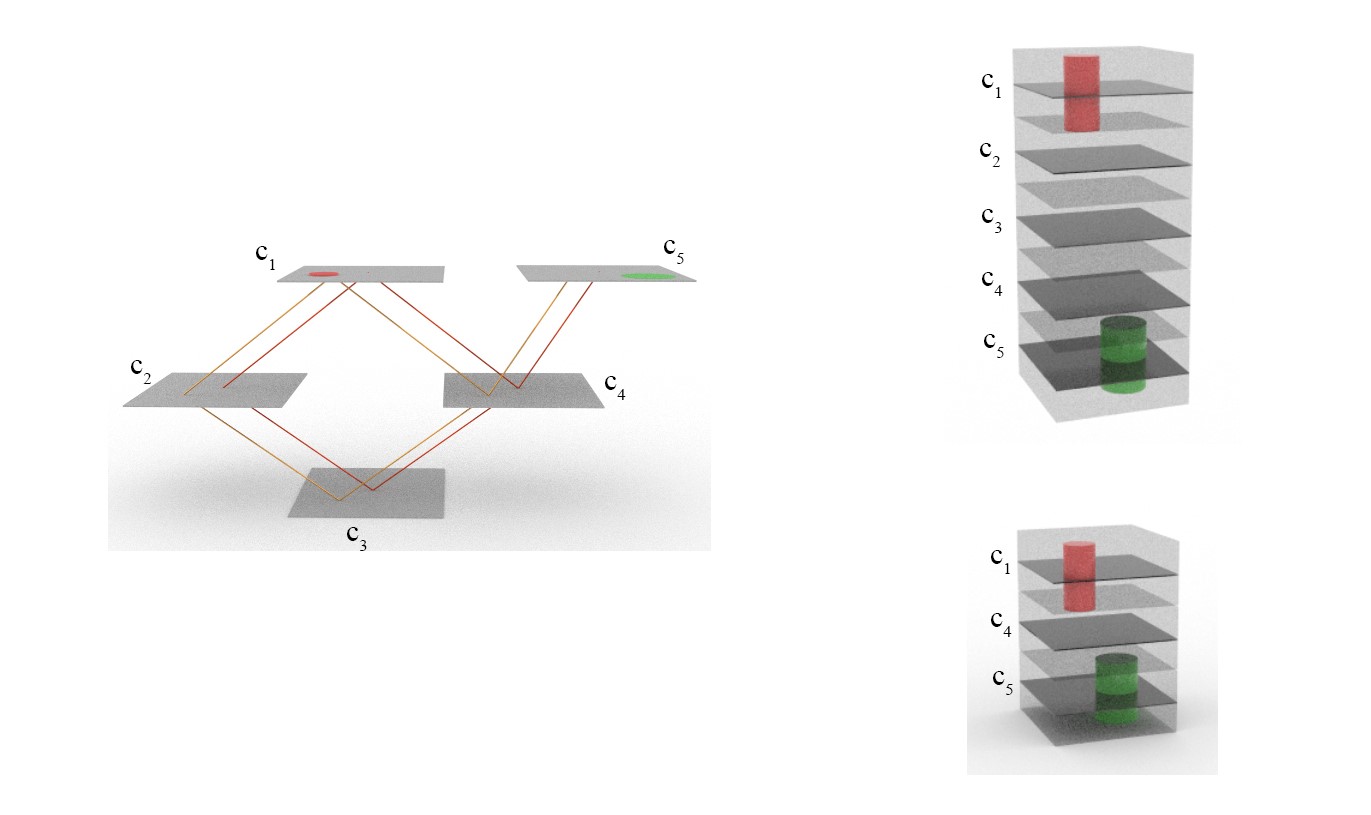}
    \caption{When we consider the transport map from the red Gaussian distribution to the green Gaussian, we consider the transport problem on all the charts that cover both Gaussian distributions. The above figure gives two of the charts.}
    \label{fig:twocharts}
\end{figure}

\begin{theorem}
For any two Gaussian cylinder-shaped distributions whose projections on each layer's range are simple Gaussian distributions denoted by $\nu_0$ and $\nu_1$ and located on layers $u$ and $w$, respectively, the optimal transport $\mathcal{W}_2$ distance between them on $\mathcal{M}$ is given by $d_{\mathcal{M}}=\mathcal{W}_2(\nu_0,\nu_1)^2+\gamma\tilde{d}_G(u,w)^2$
\end{theorem}

\begin{proof}
We consider all couplings on the manifold $\mathcal{M}$ denoted by $\Pi(\mathcal{M})$. More precisely, we consider all the possible transports on the charts in $A$ which can cover the supports of both Gaussian cylinders lifted from the two original Gaussians. Namely, we consider all the charts in $\{\mR^N\times p|p\in [G^c]_0^{uw}\}$ where $[G^c]_0^{uw}$ denotes the subset of $[G^c]_0$ of those paths contain both layer $u$ and layer $w$. With this definition, we can explicitly formulate the optimization problem:
\begin{align*}
    d_{\mathcal{M}}=&\inf_{\pi\in\Pi(\mathcal{M})} \int_{\mathcal{M}\times \mathcal{M}}||\tilde{x}-\tilde{y}||^2\pi(d\tilde{x},d\tilde{y})\\
  =&\inf_{p\in [G^c]_0^{uw}}\inf_{\pi\in\Pi^p(\mR^{N+1})} \int_{\mR\times \mR}\int_{\mR^{N}\times \mR^{N}}||x-y||^2+|z_1-z_2|^2\pi(dxdz_1,dydz_2)
\end{align*}
Here, $\Pi^p(\mR^{N+1})$ denotes the couplings in $\mR^{N+1}$ (which is homeomorphic to $\mR^N\times p$) whose marginals are the source and target Gaussian cylinders, respectively.

Further, because of the special structure of ``Gaussian cylinders,'' the first $N$ indices and the last index may be treated separately. If we denote by $\Pi_1(\mR^{N})$ the set of couplings in $\mR^{N}$ for which the two marginals are the original source and target Gaussians (which does not depend on the path $p$), and denote by $\Pi^p_2(\mR)$  the set of couplings whose two marginals are two uniform distributions located in their corresponding layers, the distance expression may be divided into two parts:
\begin{align*}
    &\inf_{p\in [G^c]_0^{uw}}\inf_{\pi_1\in\Pi_1(\mR^{N})} \int_{\mR^{N}\times \mR^{N}}||x-y||^2\pi_1(dx,dy)+\inf_{p\in [G^c]_0^{uw}}\inf_{\pi_2\in\Pi^p_2(\mR)} \int_{\mR \times \mR}|z_1-z_2|^2\pi_2(dz_1,dz_2)
\end{align*}

The second term is a simple 1D optimal transport problem between two uniform distributions which are centered at $u$ and $w$, respectively, with the same radius of thickness of each layer. The optimal transport distance between them is simply the distance between their respective centers, which is easy to calculate since both centers are located on the path $p$. To be specific, the distance is the length of the path connecting $u$ and $w$ times the thickness of each layer. Hence,
\begin{align*}
    d_{\mathcal{M}}=&\inf_{\pi_1\in\Pi_1(\mR^{N})} \int_{\mR^{N}\times \mR^{N}}||x-y||^2\pi_1(dx,dy)+\inf_{p\in [G^c]_0^{uw}}\Delta_p z^2\\
  =&\quad\ \mathcal{W}_2(\nu_0,\nu_1)^2+\gamma\tilde{d}_G(u,w)^2.
\end{align*}
Here the relative distance $\Delta_p z$ is determined by the path $p$. Moreover, $\gamma$ is introduced as a parameter for the thickness of each layer's range. We assume that the thickness of each layer is $\sqrt{\gamma}$. The minimum among all the possible paths is just $\tilde{d}_G(u,w)$, the shortest distance on the graph $G$ between vertices $u$ and $w$.
\end{proof}

Using the latter theorem, we can compute the minimum $\mathcal{W}_2$ cost of moving a source Gaussian distribution to a Gaussian target distribution. Indeed, for the $i^{\rm th}$ and $j^{\rm th}$ Gaussian cylinders on $\mathcal{M}$, we set
\begin{equation}\label{c2}
    c_2(i,j)=\mathcal{W}_2(\nu_0^i,\nu_1^j)^2+\gamma\tilde{d}_G(q_0^i,q_1^j)^2 .
\end{equation}
If we take $c_2(\cdot,\cdot)$ in (\ref{c2}) as the cost matrix and compute the Kantorovich formulation of OMT, we can derive a distance:
\begin{equation}\label{vector gmm kantorovich2}
  d_{V_2}(\rho_0,\rho_1)^2=\min_{\pi\in\Pi(p_0,p_1)}\sum_{i,j}c_2(i,j)\pi(i,j).
\end{equation}
This distance is derived from the continuous manifold, but it may be shown to be a metric for our original vector-valued distributions. See the proof in Appendix~\ref{metric_proof}.

In addition to the latter distance, we can derive the optimal transport plan $\tilde{\pi}_2^*(\cdot,\cdot)$ from the optimal solution of (\ref{vector gmm kantorovich2}). The transport plan gives the combination of weights for how the ``Gaussian cylinders'' are arranged at different orbits within each layer's range.

Based on the optimal transport plan, a geodesic (proof in Appendix~\ref{geodesic2}) on the manifold $\mathcal{M}$ may be expressed in the following manner:
\begin{equation}\label{true geodesic}
  \rho_t=\sum_{i,j}\tilde{\pi}_2^*(i,j)\nu_t^{ij}U_z(path_G(q_0^i,q_1^j,t)),
\end{equation}
where $U_z(z_0)$ is the 1D uniform distribution density function centered at $z_0$ on a path of the graph $G$. The distribution $\nu_t^{ij}U_z(path_G(q_0^i,q_1^j,t))$ at time $t$ is supported on the chart defined by the shortest path that connects $q_0^i$ to $q_1^j$ on the graph, expressed in its own local continuous coordinates for the $z$ index along that path.

For each pair of Gaussians, source and target, a deformation Gaussian cylinder moves across layers following the shortest path on the graph. When it moves across a layer boundary, the Gaussian cylinder is cut into two parts belonging to the respective ranges of two adjacent layers. Each part remains a Gaussian cylinder. Hence after projection of each layer's range, the projected distribution is still a vector GMM distribution. Now if we project (\ref{true geodesic}) onto the range of each layer, we get the following displacement interpolation in the original space:
\begin{equation}\label{vector project geodesic}
  \rho_t=\sum_{i,j}\tilde{\pi}_2^*(i,j)\nu_t^{ij}\vec{\delta}_{path_G(q_0^i,q_1^j,t)}.
\end{equation}

Note this form of displacement interpolation is very similar to (\ref{vector geodesic}) just with slightly different weights.

\textbf{Remark 1:} We should note that the way in which we define vector-valued GMM, already makes it a manifold (each layer is its own chart). However, it is impossible to define charts that contain all the possible paths, which is to say the atlas contains only the geometric information within each layer. Our manifold $\mathcal{M}$ on the other hand, has an atlas that contains all the possible paths which encode global geometric information for which we can solve the OMT problem.

\textbf{Remark 2:} This approach gives a geometric intuitive understanding of the optimal transport for vector-valued GMMs. Comparing to the cost matrix (\ref{c1}) in Approach 1, (\ref{c2}) just uses the sum of squares instead of the direct sum, in analogy to the difference between $\mathcal{W}_1$ and $\mathcal{W}_2$. It may seem complicated to consider all the possible paths on the graph in that approach. But for actual computations, we only need to consider the shortest path on the $G$ that connects two layers. Moreover, from the parameter $\gamma$ that also appears in Approach 1, we find a clear geometric meaning: it  represents the thickness of each layer.

\section{OMT for unbalanced distributions}
The unbalanced OMT problem seeks to remove the restriction that the source and target have the same total mass. The total density of the spaces may be different, which means that mass can be created or destroyed during the process of transport. Very commonly, we encounter unbalanced data related to imaging problems in which intensity is taken to be mass. In medicine, mass may be created with cell proliferation and destroyed with cell apoptosis. Normalizing both inputs to be probability distributions is not natural and may cause numerical issues and in some cases loss of information. Thus we want to find a way to directly employ unnormalized data for the OMT on Gaussian mixtures. In short, for GMMs, we want to use the general form (\ref{gmm}) but without the restriction (\ref{probability1}) i.e., source and target distributions can have different total densities.

\subsection{Unbalanced GMM optimal mass transport}
We propose here a natural method for solving the unbalanced OMT problem in the GMM framework. The idea comes from \cite{zhu2020vectorial}, where we found that by adding a source layer to the original scalar problem, and using a special set of weight parameters, we can reformulate the unbalanced OMT problem as a version of the vector-valued OMT problem. But that work was based on Benamou-Brenier's computational fluid dynamics (CFD) form of OMT \cite{BB} and a Fisher-Rao source term \cite{chizat:hal-01271981}. Here we will show that we can use the same source layer trick in the Kantorovich formulation of the OMT problem, so that it may be employed to solve the unbalanced GMM problem as well.

The idea is to add a source layer to both the initial (starting) and target distributions and put the mass difference in that layer, so that the problem becomes balanced but vector-valued. If the initial distribution has more total mass, then the mass difference is added to the source layer of the target distribution. On the other hand, if the target distribution has more total mass, then the mass difference is added to the source layer of the initial distribution. So instead of considering the $n_0$ Gaussians to $n_1$ Gaussians one-layer problem, now we consider $n_0$ Gaussians to $(n_1+1)$ Gaussians (or $(n_0+1)$ Gaussians to $n_1$ Gaussians), i.e., two-layered vector-valued problem.

We should note that we do not need to specify the distribution (a Gaussian mixture) in the source layer in the first place, but just to consider a single implicit Gaussian distribution in that layer. In the $c(\cdot,\cdot)$ matrix of size $n_0\times (n_1+1)$ or $(n_0+1)\times n_1$,  the sub-matrix $c(1:n_0,1:n_1)$ may be computed as before (\ref{gmm_c}). We put $\gamma$ in the last row (or column), which is to say the cost of using source for unit mass is $\gamma$. Note that $\gamma$ is a parameter to control the weight of the source in the total distance formula.

After we have solved this extended Kantorovich problem, we can compute a displacement interpolation, since we have the optimal solution of $\pi^*(\cdot,\cdot)$. Note that the meaning of the last row (or the last column) is the amount of source the $j^{\rm th}$ Gaussian uses in the target distribution (or $i^{\rm th}$ Gaussian uses in the starting distribution). So we can add exactly the necessary amount of mass as a weighted same-shaped Gaussian in the source layer for each corresponding Gaussian in the original space. Then for the two explicit 2-vector GMMs, we can compute the displacement interpolations between them. The interpolation of only the original channel is the displacement interpolation for the two unbalanced GMMs.

\textbf{Remark} Actually we can use this approach more generally than the GMM or source layer case. Indeed, for any two unbalanced discrete distributions, one can add an implicit ``source node'' to the source or target distribution in order to make the two extended distributions balanced. Then if we consider the extended Kantorovich problem, we can get an OMT related ``distance'' between these two unbalanced distributions.

\subsection{Vector-valued GMM OMT and unbalanced vector-valued GMM problem}
We can treat the unbalanced vector-valued GMM problem in this framework as well. The problem of computing the OMT distance and the associated interpolation for color image data falls into this category, which makes it of great practical usefulness.

As before, the idea is to add a new source layer for both the source and target vector-valued GMM structures. The source layer is made to connect with each of the other layers. If we again put an implicit Gaussian function in that layer, we can get a balanced but extended version of vector GMM OMT problem, which can be solved through the scheme from Section~\ref{another approach}. The interpolation of the original part of two extended structures (the part other than the source layers) is then the displacement interpolation for the two unbalanced vector-valued GMMs.

\section{Numerical results}
Here we describe some numerical results of our proposed method.

\subsection{1D examples: approach 1 vs approach 2}
We consider the simplest 1D 2-channel vector GMM problem. Figure~\ref{fig:compare2approach} shows the distributions of the initial and target vector distributions.

\begin{figure}[H]
  \centering
  \includegraphics[width=0.2\linewidth]{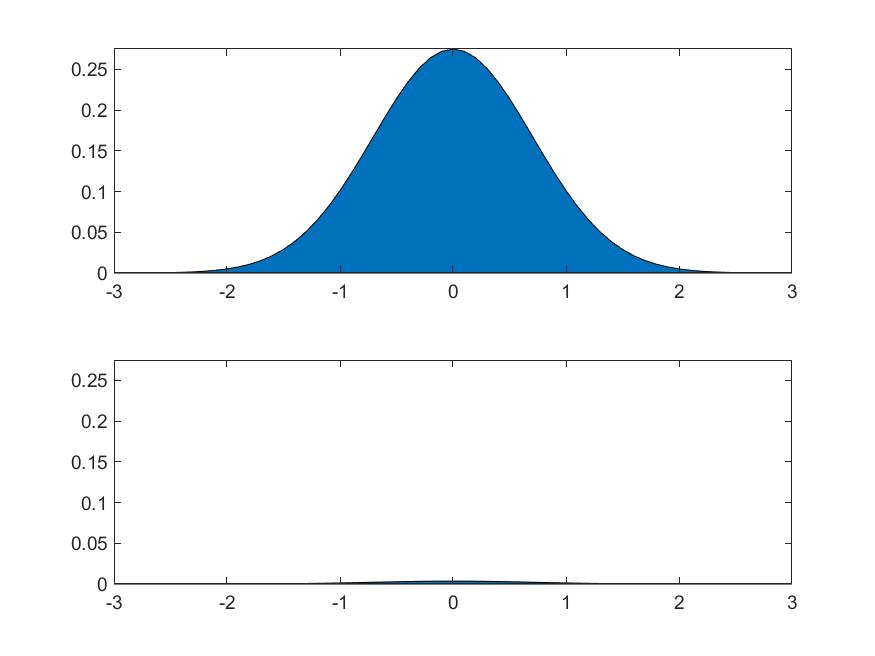}
  \includegraphics[width=0.2\linewidth]{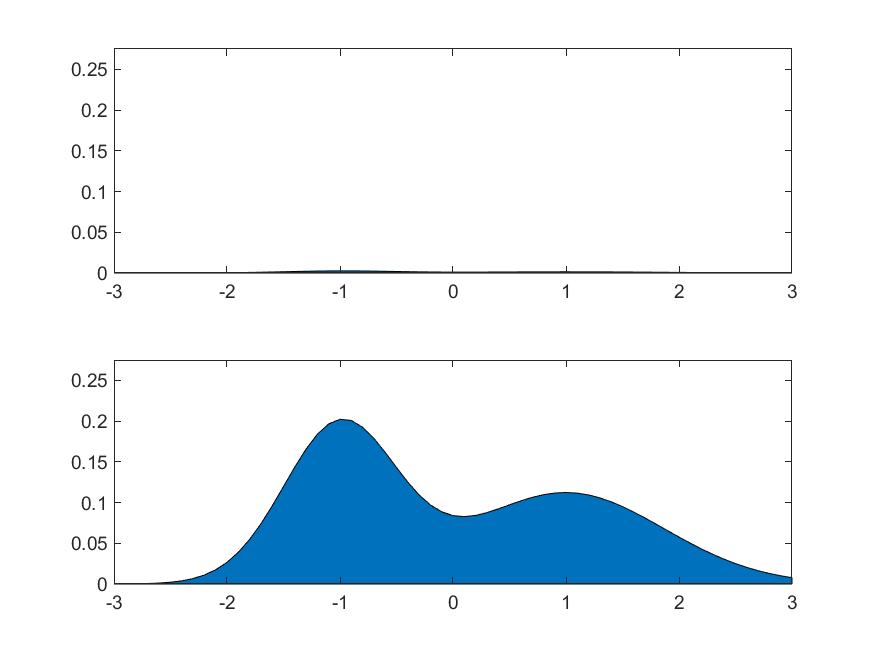}
  \caption{1D 2-channel example: starting and target vector-valued distributions}\label{fig:compare2approach}
\end{figure}

We tested the same example via both our Approach 1 (see Figure \ref{fig:density-1}) and Approach 2 (see Figure \ref{fig:density0}).
\begin{figure}[H]
  \centering
  \foreach \t in {1,3,5,7,9}{
  \includegraphics[width=0.18\linewidth]{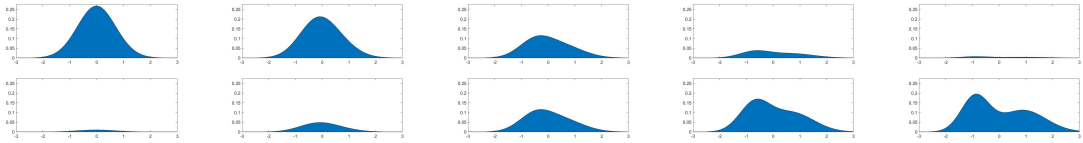}
  }
  \caption{1D 2-channel Example 1 - using approach 1: vector-valued distributions over time}\label{fig:density-1}
\end{figure}

\begin{figure}[H]
  \centering
  \foreach \t in {1,3,5,7,9}{
  \includegraphics[width=0.18\linewidth]{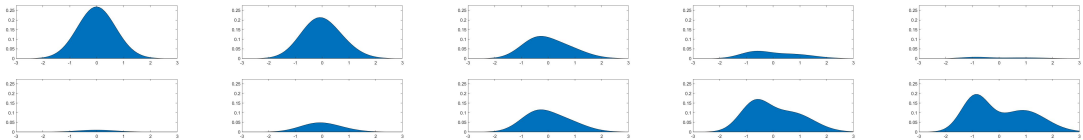}
  }
  \caption{1D 2-channel example 2 - using approach 2: vector-valued distributions over time}\label{fig:density0}
\end{figure}

The two sets of results look very similar. For most of the examples that we have tested, we cannot find discernible differences between these two approaches. This is due to the similar and correlated formulations of the two given cost matrices. We will thus use the formulation of Approach 2 in all of the following examples.

\subsection{1D examples: graph structure}
We consider the simplest 1D 3-channel problem, which is the simplest example with a nontrivial graph structure. Figure~\ref{fig:st1} shows the distributions of the initial (starting) and target vector distributions. The distributions in their 3 channels are explicitly and separately plotted. All the mass of the initial vector distribution is located in the first channel (one Gaussian), and all the mass of the target vector distribution (two Gaussians) is located in the third channel.

We tested two different graph structures for the vector-valued distribution problem. In Figure~\ref{fig:density1}, we show the interpolation results whose second channel is connected to the first and third ones, but whose first and third channels are not directly connected. In Figure~\ref{fig:density2}, we show the interpolation results in which three layers are fully connected.

We can see from these results that the mass transfer in \ref{fig:density1} goes through the second layer, while in \ref{fig:density2} mass in the first layer can directly transfer to the third layer.

\begin{figure}[H]
  \centering
  \includegraphics[width=0.2\linewidth]{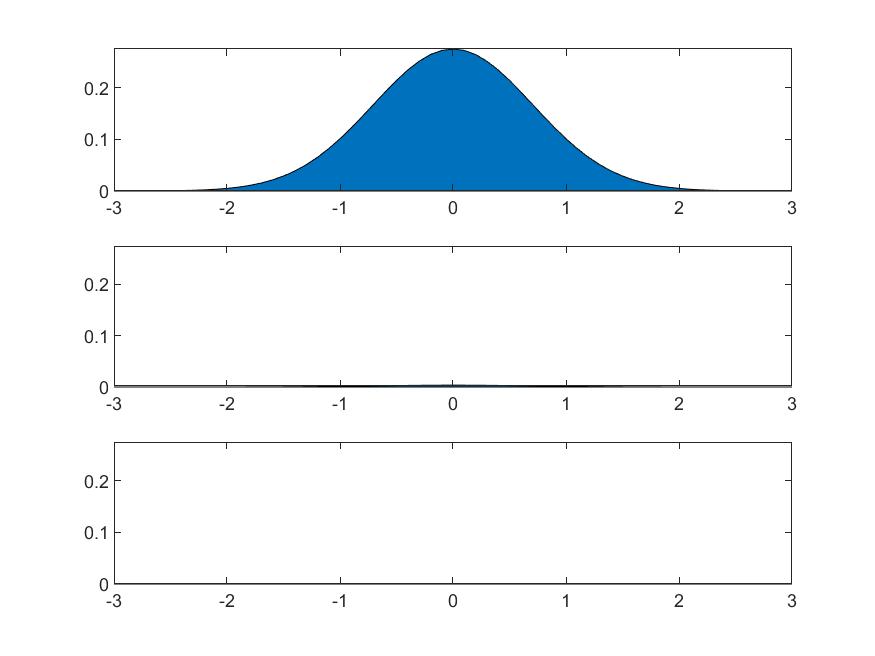}
  \includegraphics[width=0.2\linewidth]{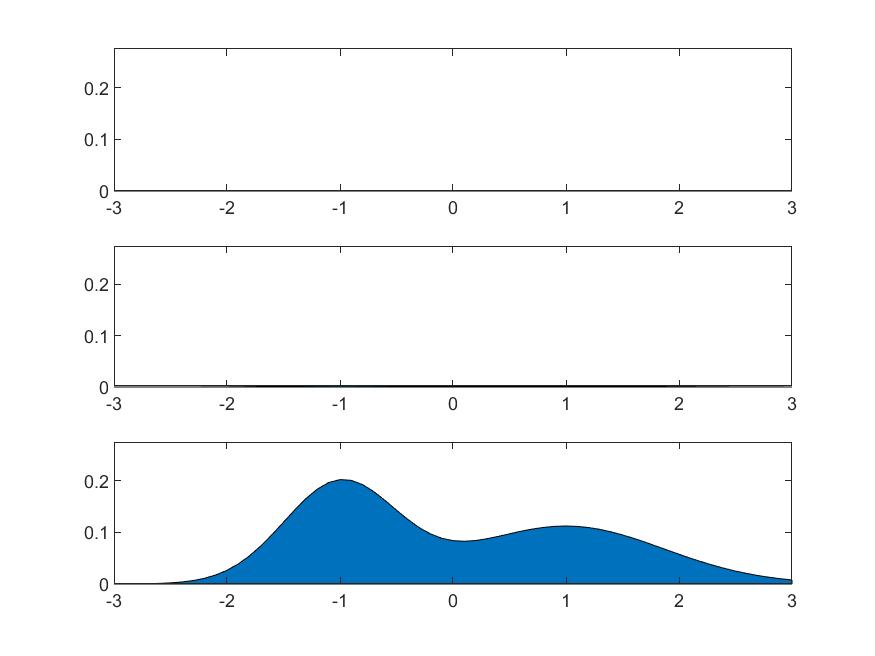}
  \caption{1D 3 channels example: starting and target vector-valued distributions}\label{fig:st1}
\end{figure}
\begin{figure}[H]
  \centering
  \foreach \t in {1,3,5,7,9}{
  \includegraphics[width=0.18\linewidth]{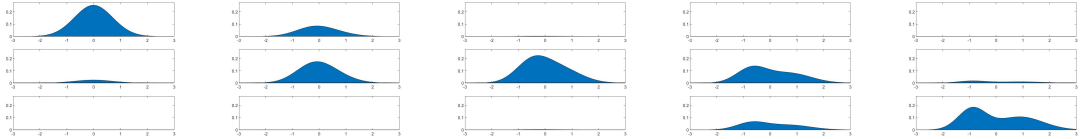}
  }
  \caption{1D 3-channel example 1 - not fully connected graph: vector-valued distributions over time}\label{fig:density1}
\end{figure}

\begin{figure}[H]
  \centering
  \foreach \t in {1,3,5,7,9}{
  \includegraphics[width=0.18\linewidth]{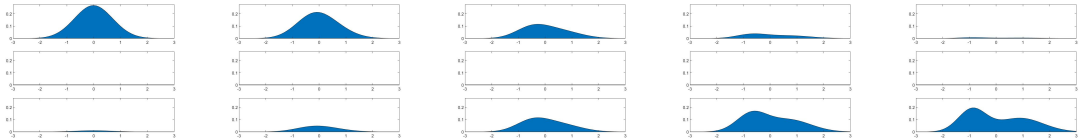}
  }
  \caption{1D 3-channel example 2 - fully connected graph: vector-valued distributions over time}\label{fig:density2}
\end{figure}

\subsection{2D examples: parameter $\gamma$}\label{2deg}

Now we consider 2D 3-vector distributions, and we take a graph which is not completely connected: the green channel is connected to the red channel and to the blue channel, but the red and blue channels are not directly connected. This is exactly the model for RGB 3-channel images. So the following vector-valued distributions are plotted as color images. We still use approach 2 method for this numerical test.

There are two different kinds of mass transfer: within layers and between layers. The parameter $\gamma$ controls the amount of mass that moves across layers. Indeed, with $\gamma$ large, the cost of moving mass between layers is large. By using different values of $\gamma$, we may get very different results.

We use a 2-Gaussian to 2-Gaussian example to show the effect of different values of $\gamma$. Both the initial and target vector-valued distributions have one Gaussian in the red layer and one in the blue layer. But the location of the balls are interchanged (see Figure~\ref{fig:st2}). With $\gamma$ large, the two Gaussians just move within their own respective layers (see Figure~\ref{fig:density3}). With $\gamma$ small, the red ball goes through green layer to blue layer while the blue ball takes the opposite path (see Figure~\ref{fig:density4}).

\begin{figure}[H]
  \centering
  \includegraphics[width=0.2\linewidth]{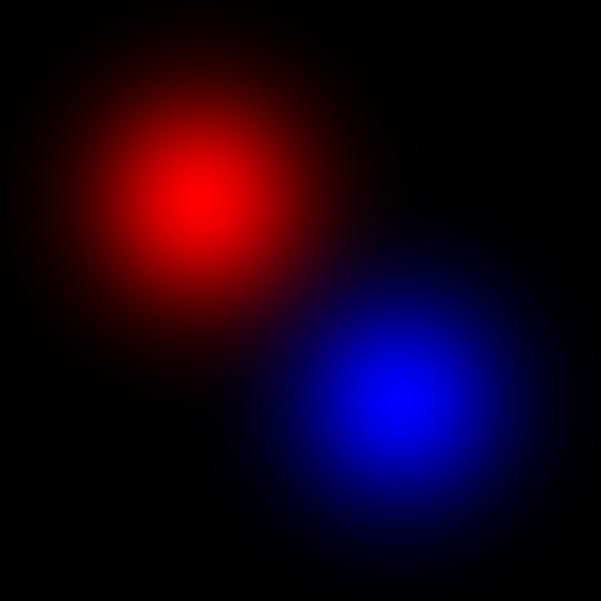}
  \includegraphics[width=0.2\linewidth]{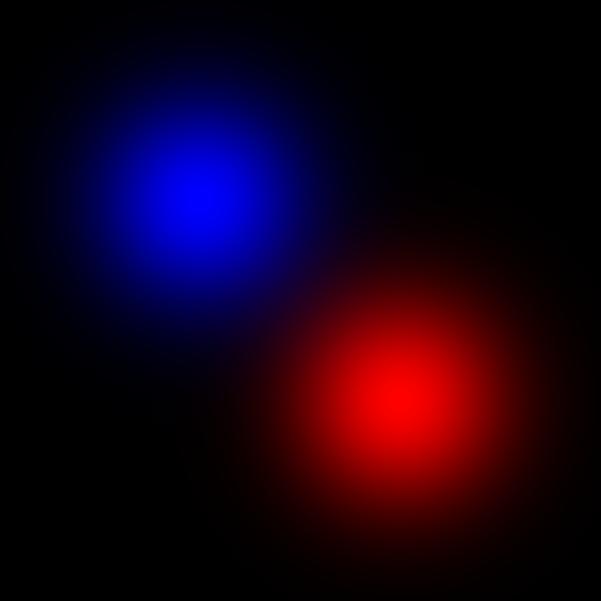}
  \caption{2D 3 channels example: starting and target vector distributions}\label{fig:st2}
\end{figure}
\begin{figure}[H]
  \centering
  \foreach \t in {1,3,5,7,9}{
  \includegraphics[width=0.18\linewidth]{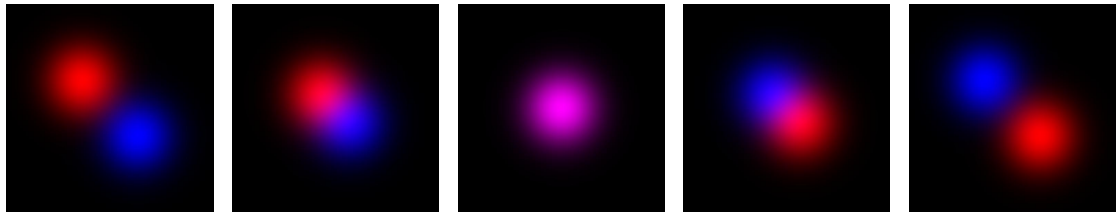}
  }
  \caption{2D example 1 - $\gamma$ large: vector-valued distributions over time}\label{fig:density3}
\end{figure}

\begin{figure}[H]
  \centering
  \foreach \t in {1,3,5,7,9}{
  \includegraphics[width=0.18\linewidth]{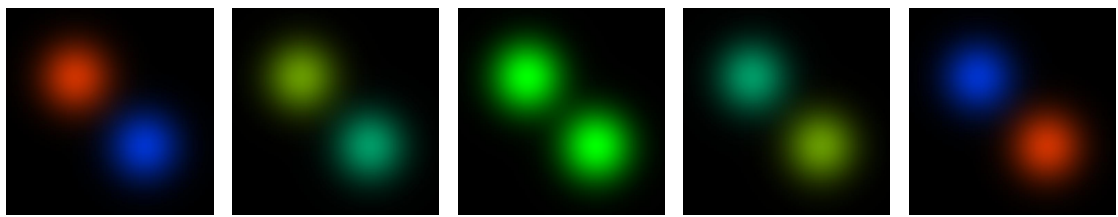}
  }
  \caption{2D example 2 - $\gamma$ small: vector-valued distributions over time}\label{fig:density4}
\end{figure}



\subsection{Realistic image data}
\subsubsection{Moon}

We tested our method on real-world data, namely, moon images. We first fitted GMMs to the images, and then applied our method.

Figure~\ref{fig:moon} shows the moon data on which we ran our algorithm. We fitted each layer with 400 Gaussians. The results are shown in Figure~\ref{fig:fittedmoon}. Figure~\ref{fig:intermoon} gives the interpolation between the two fitted GMMs.

\begin{figure}[H]
  \centering
  \includegraphics[width=0.2\linewidth]{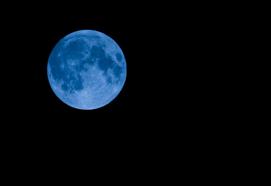}
  \includegraphics[width=0.2\linewidth]{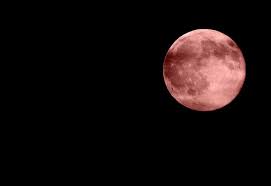}
  \caption{Original moon image: the left one is a blue moon and the right one is a pink moon}\label{fig:moon}
\end{figure}
\begin{figure}[H]
  \centering
  \includegraphics[width=0.2\linewidth]{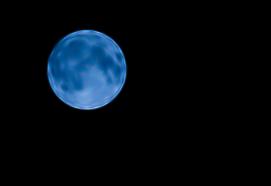}
  \includegraphics[width=0.2\linewidth]{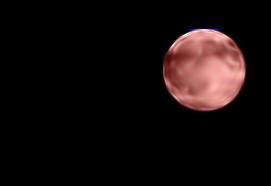}
  \caption{Fitted GMMs as source and target}\label{fig:fittedmoon}
\end{figure}
\begin{figure}[H]
  \centering
  \foreach \t in {10,8,6,4,2,0}{
  \includegraphics[width=0.15\linewidth]{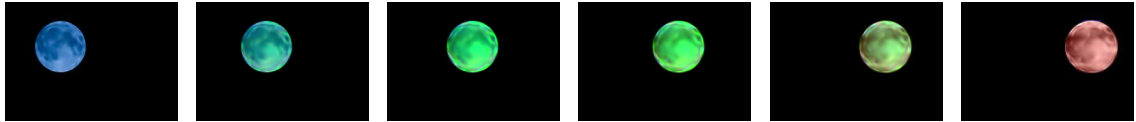}
  }
  \caption{Image example: vector-valued GMM geodesic path}\label{fig:intermoon}
\end{figure}

\subsubsection{Nebula}
In addition to the moon imagery, we tested our methodology on more complicated data, namely, nebulae. We obtained Figure~\ref{fig:nebula} from the NASA website \cite{jenner_2020}. This image compares two drastically different portraits of the Stingray nebula captured by NASA's Hubble Space Telescope 20 years apart. We fitted each layer by 80 Gaussians. The fitting results are shown in Figure~\ref{fig:fitnebula}. Figure~\ref{fig:fitnebula} shows the interpolation.

\begin{figure}[H]
    \centering
    \includegraphics[width=0.5\linewidth]{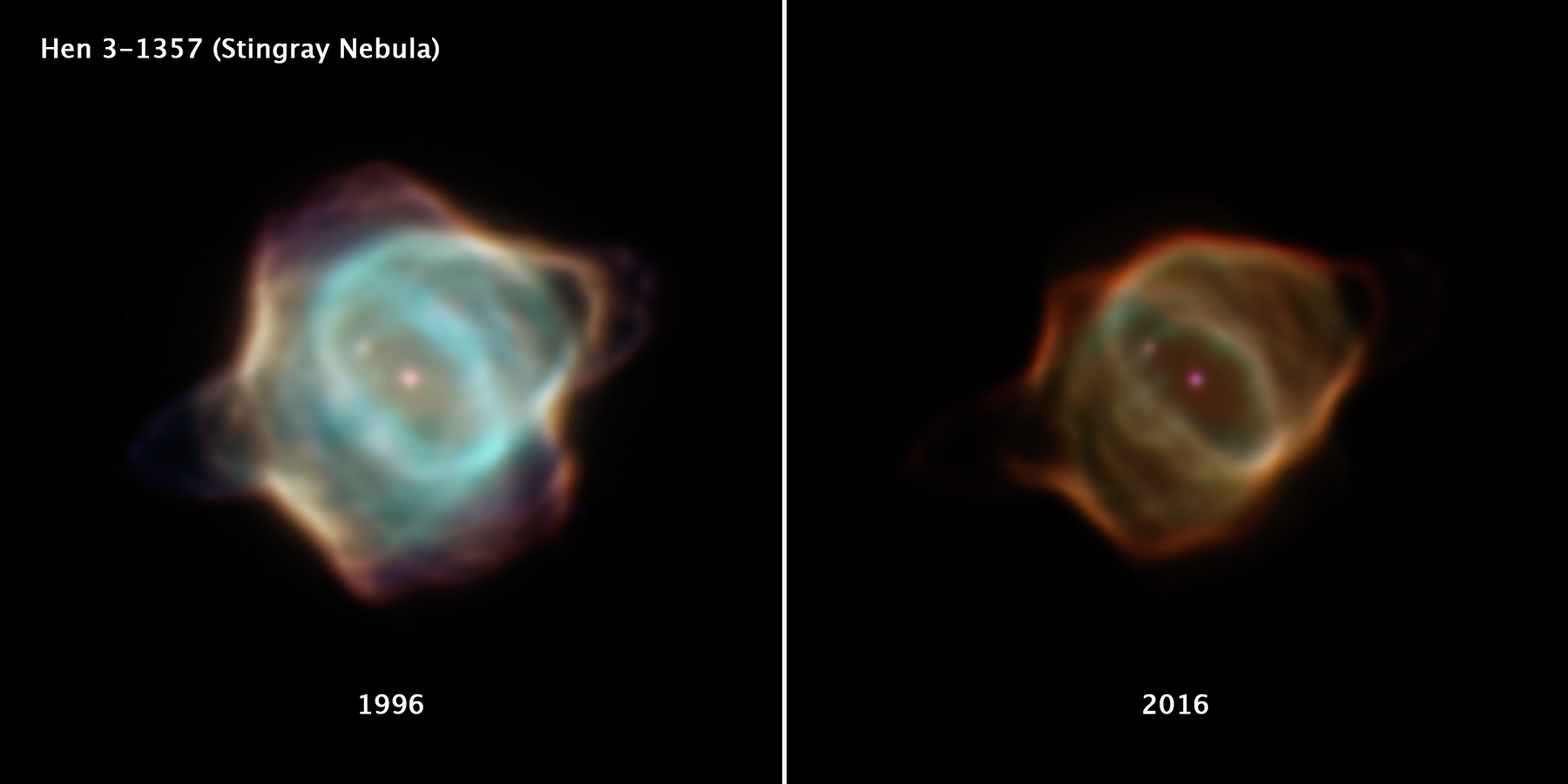}
    \caption{Stingray Nebula in 1996 and 2016}
    \label{fig:nebula}
\end{figure}

\begin{figure}[H]
    \centering
    \includegraphics[width=0.25\linewidth]{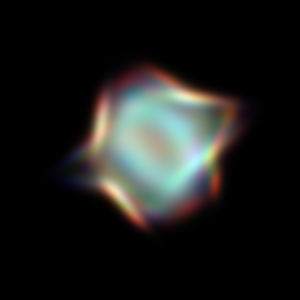}
    \includegraphics[width=0.25\linewidth]{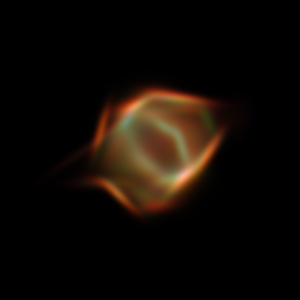}
    \caption{Fitted GMMs of the nebula image data}
    \label{fig:fitnebula}
\end{figure}

\begin{figure}[H]
    \centering
    \foreach \t in {0,2,4,6,8,10}{
  \includegraphics[width=0.15\linewidth]{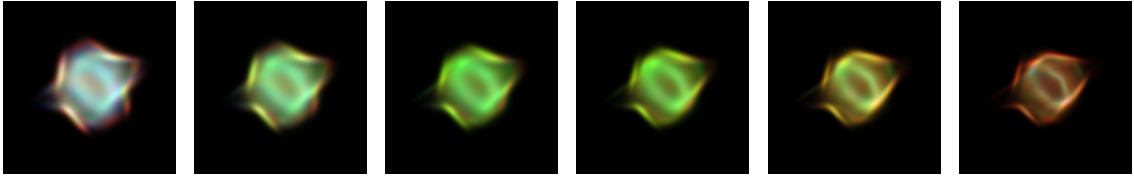}
  }
    \caption{Image example: vector-valued GMM geodesic path}
    \label{fig:internebula}
\end{figure}


\subsection{Fonts}
As an additional interesting application, we considered the transport between two different fonts, which may, for example, be used as an animation tool in PowerPoint. 

In Figure~\ref{fig::pix_croiss}, we show two different fonts of letters "MATH". These two fonts use 80 Gaussians and 19 Gaussians respectively. Figure~\ref{fig:interfont} shows the transformation. 

\begin{figure}[H]
    \centering
    \includegraphics[width=0.2\linewidth]{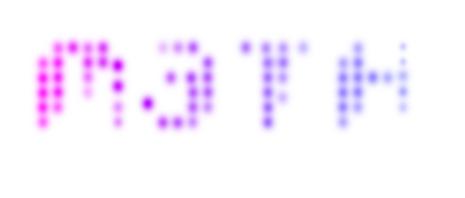}
    \hspace{3cm}
    \includegraphics[width=0.2\linewidth]{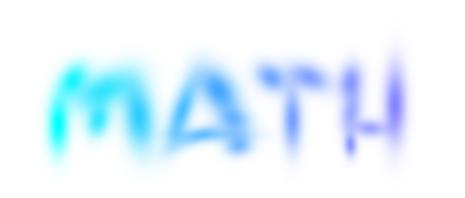}
    \caption{Two different fonts. Left: 80 Gaussians; Right: 19 Gaussians}
    \label{fig::pix_croiss}
\end{figure}

\begin{figure}[H]
  \centering
  \foreach \t in {0,2,4,6,8,10}{
  \includegraphics[width=0.15\linewidth]{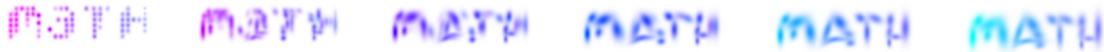}
  }
  \caption{Font transformation}
  \label{fig:interfont}
\end{figure}

We also tested the transformation between two different words. In Figure~\ref{fig:ejam2math}, we show how ``math'' be transformed in to ``game.'' Both of them contain 80 Gaussians. 

\begin{figure}[H]
    \centering
    \foreach \t in {0,2,4,6,8,10}{
  \includegraphics[width=0.15\linewidth]{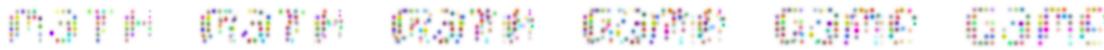}
  }
    \caption{Transformation between words ``math'' and ``game''}
    \label{fig:ejam2math}
\end{figure}

It is clear from the latter example that the proposed transport map preserves the Gaussian mixture structure. The transformation moves and rearranges the colored balls.

\subsection{Unbalanced GMM OMT}
We also tested our proposed methodology on unbalanced GMM data. The total mass of two Gaussians on the right in the initial (source) distribution is equal to the mass of the large Gaussian at the center of the target distribution. The mass of the two small Gaussians on the left of the target GMM equals the mass difference (See Figure~\ref{fig:st5}).

As expected, the two Gaussians in the initial distribution move together and combine as the larger Gaussian at the center of target distribution, while the presence of two small Gaussians on the left is due to the injection of source (see Figure~\ref{fig:density6}). It is easy to see from the surface plots of both the original and source layers as surfaces in 3D, that there are two Gaussian distributions in the source layer as the source of the mass difference in the original layer (see Figure~\ref{fig:density7}).

\begin{figure}[H]
  \centering
  \includegraphics[width=0.2\linewidth]{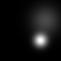}
  \includegraphics[width=0.2\linewidth]{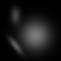}
  \caption{Fitted GMMs as starting and target}\label{fig:st5}
\end{figure}
\begin{figure}[H]
  \centering
  \foreach \t in {1,3,5,7,9}{
  \includegraphics[width=0.18\linewidth]{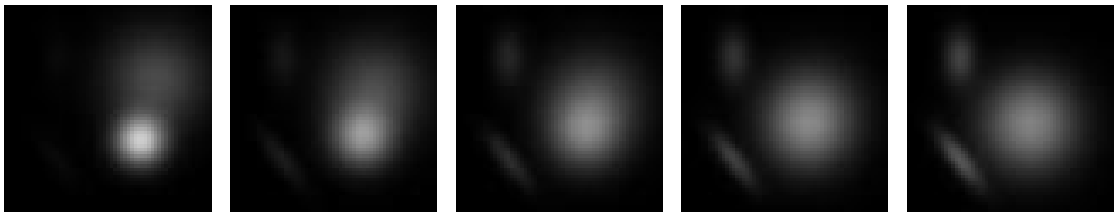}
  }
  \caption{Unbalanced example: vector-valued distributions over time}\label{fig:density6}
\end{figure}
\begin{figure}[H]
  \centering
\foreach \t in {1,3,5,7,9}{
  \includegraphics[width=0.18\linewidth]{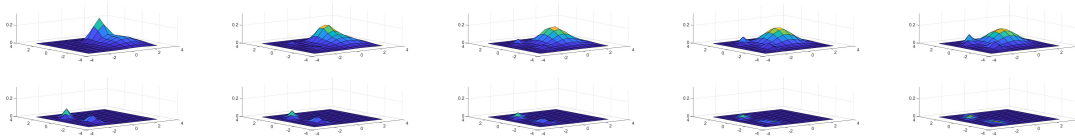}
  }
  \caption{Unbalanced example: vector distributions and source layer plotted as surfaces}\label{fig:density7}
\end{figure}

\section{Conclusion}
This work focuses on the optimal transport for vector-valued GMMs, which is a structured version of vector-valued OMT. As an extension of \cite{Chen2019}, we defined a distance and geodesic path in the vector-valued case. To the best of our knowledge, the present work is the \textbf{\emph{first}} to employ a manifold-based approach to the problem of GMM vector-valued data. Simply applying manifold-valued OMT to vector-valued distributions while preserving the vector GMM structure is not completely straightforward. In fact, just combining the vector-valued case in which the layers are connected by a general graph structure \cite{vectorvalued} and the GMM metric \cite{Chen2019,delon_gmm} via adding constraints to the appropriate set of joint probability distributions may not work. See Appendix A below for all of the details. Thus one needs a manifold-based approach in the present situation, which we have shown easily extends to the unbalanced case. In particular, we have extended the approach of transforming the unbalanced scalar OMT problem to the balanced vector-valued problem from CFD \cite{zhu2020vectorial} to a Kantorovich formulation in this work. Preserving the GMM structure along a transport path both in the balanced and unbalanced cases may have broad applications given the prevalence of such models in many areas of engineering, computer science, and machine learning \cite{McLachlan2000}.

 The proposed transport is useful, because of its speed advantage and unique ability to preserve structure. This paper just investigates Gaussian mixture case, but it is quite straightforward to apply our framework for other mixture models.  We are planning on applying our methodology to the analysis of medical imagery and other appropriate vector-valued distributions including multi-omic data.

\section{Acknowledgements}

This study was supported by AFOSR grants (FA9550-17-1-0435, FA9550-20-1-0029), NIH grant (R01-AG048769), and a grant from Breast Cancer Research Foundation (grant BCRF-17-193).

\begin{appendices}
\section{Approach 0: Direct vectorial generalization of GMM}\label{modify_PI}

Defining a ``correct'' vectorial notion of OMT in the GMM case is not completely straightforward. Here we show why the simplest scheme may fail.
Letting $\Pi(p_0,p_1)$ denote the set of joint probabilities with given marginals $p_0$ and $p_1$, the most straightforward idea is to add certain constraints on $\Pi$. Namely, we add the following constraints based on the graph structure:
\begin{equation}
  \Pi_G(p_0,p_1)= \left \{ \pi\in\mR_+^{n_0\times n_1}|\pi \vec{1}_{n_1}=p_0,\ \pi^T \vec{1}_{n_0}=p_1, \pi(i,j)=0\ \text{if}\ (q^i_0,q^j_1)\notin E(G)\cup\{(v,v)|v\in V(G)\} \right \}.
\end{equation}
Thus transport is allowed between two Gaussians only when they are in the same channel or they located in adjacent channels.

Accordingly, we can define a ``distance:''
\begin{equation}\label{straight vector gmm}
    d_{V_0}(\rho_0,\rho_1)^2=\inf_{\pi\in\Pi_G(p_0,p_1)}\sum_{i,j}\pi(i,j)\mathcal{W}_2(\nu^{i}_0,\nu^{j}_1)^2.
\end{equation}

This ``distance'' is a pseudo-metric, which is to say, $d_{V_0}(\cdot,\cdot)$ satisfies all the other metric conditions, but may be zero for two different distributions. Indeed, $d_{V_0}(\cdot,\cdot)$ is zero if one vector GMM is a permutation of channels of the other GMM. One may prove that
$d_{V_0}(\cdot,\cdot)\geq 0$ and satisfies triangle inequality. We do not include the proofs because of the limitations described below.

\subsection{Limitations of  approach 0}\label{limitations}

There are several problems with this approach, which led us to propose the other approaches in the main text. Here are some of the problems:
\begin{itemize}
    \item There may be no solution, since $\Pi_G(p_0,p_1)$ may be empty.
    \item The transport can only go through two connected channels.
    \item The cost of mass transport between channels is not included.
    \item As discussed above, the proposed ''distance'' may be $0$ for two different marginals.
\end{itemize}
For these reasons, we did not pursue this straightforward direct approach.

\section{Proof of Theorem \ref{metric}}\label{new distance}
\begin{proof}
Clearly, $d_{V_1}(\rho_0,\rho_1)\ge0$, $\forall \rho_0,\rho_1\in\mathcal{G}(\mathbb{R}^N)$ and $d_{V_1}(\rho_0,\rho_1)=0$ if and only if $\rho_0=\rho_1$.

We now prove the triangle inequality $d_{V_1}(\rho_0,\rho_2)\le d_{V_1}(\rho_0,\rho_1)+d_{V_1}(\rho_1,\rho_2)$. First we denote $$\pi_{02}(i,k)=\sum_{j=1}^{n_1}\frac{\pi_{01}(i,j)\pi_{12}(j,k)}{p^{j}_1}.$$
 As $\pi_{02}$ is a joint distribution with marginals $\rho_0$ and $\rho_2$, we have
    \begin{align*}
        d_{V_1}(\rho_0,\rho_2)&\le\sum_{i,k}\pi_{02}(i,k)[\mathcal{W}_2(\nu_0^i,\nu_2^k)+\gamma \tilde{d}_G(q_0(i),q_2(k))]\\
        &\le \sum_{i,j,k}\frac{\pi_{01}(i,j)\pi_{12}(j,k)}{p^{j}_1}[\mathcal{W}_2(\nu_0^i,\nu_1^j)+\mathcal{W}_2(\nu_1^j,\nu_2^k)+\gamma \tilde{d}_G(q_0(i),q_1(j))+\gamma \tilde{d}_G(q_1(j),q_2(k))]\\
        &=\sum_{i,j}\frac{\pi_{01}(i,j)p^{j}_1}{p^{j}_1}[\mathcal{W}_2(\nu_0^i,\nu_1^j)+\gamma \tilde{d}_G(q_0(i),q_1(j))]\\
        &+\sum_{j,k}\frac{p^j_1\pi_{12}(j,k)}{p^{j}_1}[\mathcal{W}_2(\nu_1^j,\nu_2^k)+\gamma \tilde{d}_G(q_1(j),q_2(k))]\\
        &=d_{V_1}(\rho_0,\rho_1)+d_{V_1}(\rho_1,\rho_2)
    \end{align*}
\end{proof}

\section{Proof of Theorem \ref{geo}}\label{geodesic_proof}
    \begin{proof}
    $$\rho_t=\sum_{i,j}\tilde{\pi}_1^*(i,j)\nu_t^{ij}\vec{\delta}_{path_G(q_0^i,q_1^j,t)}$$

    Locally, let $\bar{d}=\tilde{d}_G(q_0(i),q_i(j))$. Rename the vertices on this shortest path in order $0-1-\cdots-\bar{d}$. Denote $a=t\bar{d}-\lfloor t\bar{d}\rfloor$, $b=s\bar{d}-\lfloor s\bar{d}\rfloor$, $r=\lfloor t\bar{d}\rfloor-\lfloor s\bar{d}\rfloor$. Using this notation, the displacement interpolation may be expressed as:

    $$\rho_t=\sum_{i,j}\tilde{\pi}_1^*(i,j)\nu_t^{ij}((1-a)\vec{\delta}_{\lfloor t\bar{d}\rfloor}+a\vec{\delta}_{\lfloor t\bar{d}\rfloor+1})$$

    $$\rho_s=\sum_{i,j}\tilde{\pi}_1^*(i,j)\nu_t^{ij}((1-b)\vec{\delta}_{\lfloor s\bar{d}\rfloor}+b\vec{\delta}_{\lfloor s\bar{d}\rfloor+1})$$

    Then we have
    \begin{align*}
        d_{V_1}(\rho_s,\rho_t)&\le\sum_{i,j}\tilde{\pi}_1^*(i,j)[W_2(\nu_s^{ij},\nu_t^{ij})+\gamma(r(1-a)(1-b)+rab+(r-1)(1-a)b+(r+1)(1-b)a)]\\
        &=\sum_{i,j}\tilde{\pi}_1^*(i,j)[W_2(\nu_s^{ij},\nu_t^{ij})+\gamma(r+a-b)]\\
        &=\sum_{i,j}\tilde{\pi}_1^*(i,j)[(t-s)W_2(\nu_0^{i},\nu_1^{j})+\gamma(t-s)\bar{d}\\
        &=(t-s)\sum_{i,j}\tilde{\pi}_1^*(i,j)[W_2(\nu_0^{i},\nu_1^{j})+\gamma\tilde{d}_G(q_0(i),q_i(j))]\\
        &=(t-s)d_{V_1}(\rho_0,\rho_1)
    \end{align*}
    By the triangle inequality,
    \begin{align*}
    d_{V_1}(\rho_0,\rho_1)&\le d_{V_1}(\rho_0,\rho_s)+d_{V_1}(\rho_s,\rho_t)+d_{V_1}(\rho_t,\rho_1)\\
    &\le sd_{V_1}(\rho_0,\rho_1)+(t-s)d_{V_1}(\rho_0,\rho_1)+(1-t)d_{V_1}(\rho_0,\rho_1)\\
    &=d_{V_1}(\rho_0,\rho_1).
\end{align*}

These two inequalities give the result $d_{V_1}(\mu_s,\mu_t)=(t-s)d_{V_1}(\mu_0,\mu_1)$.
\end{proof}

\section{Metric $d_{V_2}(\cdot,\cdot)$}\label{metric_proof}
\begin{theorem}\label{metric2}
$d_{V_2}(\cdot,\cdot)$ defines a metric on $\mathcal{G}(\mR^N\times G)$
\end{theorem}
Clearly, $d_{V_2}(\rho_0,\rho_1)\ge0$ for any $\rho_0,\rho_1\in\mathcal{G}(\mR^N\times G)$.

And clearly, $d_{V_2}(\rho_0,\rho_1)=0$ if $\rho_0=\rho_1$. If $\rho_0\neq\rho_1$ and $d_{V_2}(\rho_0,\rho_1)=0$, the solution of Kantorovich problem $\Pi(x,y)=0$, a.e, because $c(x,y)\geq 0\ a.e$, $c(x,y)=0$ if and only if $x=y$ which is a zero measure set. So
$d_{V_2}(\rho_0,\rho_1)=0$ if and only if $\rho_0=\rho_1$ a.e. We next prove the triangular inequality, namely, $$d_{V_2}(\rho_0,\rho_1)+d_{V_2}(\rho_1,\rho_2)\ge d_{V_2}(\rho_0,\rho_2).$$

Before proving the triangle inequality, we first need the following lemma:
\begin{lemma}
\begin{equation}
    \sqrt{\sum_{i=1}^n(a_i+b_i)^2+(c_i+d_i)^2}\le\sqrt{\sum_{i=1}^n a_i^2+c_i^2}+\sqrt{\sum_{i=1}^n b_i^2+d_i^2},
\end{equation}
$\forall\ a_i,b_i,c_i,d_i\in\mR,\ i=1,...,n$ and $n\geq 1$.
\end{lemma}
\begin{proof}
    \begin{align*}
        \mbox{LHS}^2&=\sum(a_i^2+b_i^2+c_i^2+d_i^2+2a_i b_i+2c_i d_i)\\
        &\leq\sum(a_i^2+b_i^2+c_i^2+d_i^2)+2(\sqrt{\sum a_i^2\sum b_i^2}+\sqrt{\sum c_i^2\sum d_i^2}),
    \end{align*}
    \begin{align*}
        (\sqrt{\sum a_i^2\sum b_i^2}+\sqrt{\sum c_i^2\sum d_i^2})^2&=\sum a_i^2\sum b_i^2+\sum c_i^2\sum d_i^2+2\sqrt{\sum a_i^2\sum b_i^2\sum c_i^2\sum d_i^2}\\
        &\le\sum a_i^2\sum b_i^2+\sum c_i^2\sum d_i^2+\sum a_i^2\sum d_i^2+\sum b_i^2\sum c_i^2\\
        &=(\sum a_i^2+\sum c_i^2)(\sum b_i^2+\sum d_i^2).
    \end{align*}

By combining the above two inequalities, we have
\begin{align*}
        \mbox{LHS}^2&\le \sum (a_i^2+b_i^2+c_i^2+d_i^2)+2\sqrt{(\sum a_i^2+\sum c_i^2)(\sum b_i^2+\sum d_i^2)}\\
        &=(\sqrt{\sum a_i^2+c_i^2}+\sqrt{\sum b_i^2+d_i^2})^2\\
        &=\mbox{RHS}^2.
    \end{align*}
\end{proof}

Next we prove the triangle inequality.
Similar to the proof in Appendix~\ref{new distance}, we define $$\pi_{02}=\sum_{j}\frac{\pi_{01}(i,j)\pi_{12}(j,k)}{p_1^j}.$$
Since $\pi_{02}\in\Pi(p_0,p_2)$, we have
\begin{align*}
    &d_{V_2}(\rho_0,\rho_2)\le\sqrt{\sum_{i,k}\pi_{02}(i,k)[W_2(\nu_0^i,\nu_2^k)^2+\gamma\tilde{d}_G(q_0^i,q_2^k)^2}]\\
    &\le \sqrt{\sum_{i,j,k}\frac{\pi_{01}(i,j)\pi_{12}(j,k)}{p^{j}_1}[(\mathcal{W}_2(\nu_0^{i},\nu_1^{j})+\mathcal{W}_2(\nu_1^{j},\nu_2^{k}))^2+\gamma(\tilde{d}_G(q_0^i,q_1^j)+\tilde{d}_G(q_1^j,q_2^k))^2]}\\
    &\le \sqrt{\sum_{i,j,k}\frac{\pi_{01}(i,j)\pi_{12}(j,k)}{p^{j}_1}[\mathcal{W}_2(\nu_0^{i},\nu_1^{j})^2+\gamma\tilde{d}_G(q_0^i,q_1^j)^2]}+\sqrt{\sum_{i,j,k}\frac{\pi_{01}(i,j)\pi_{12}(j,k)}{p^{j}_1}[\mathcal{W}_2(\nu_1^{j},\nu_2^{k})^2+\gamma\tilde{d}_G(q_1^j,q_2^k)^2]}\\
    &=\sqrt{\sum_{i,j}\pi_{01}(i,j)[\mathcal{W}_2(\nu_0^{i},\nu_1^{j})^2+\gamma\tilde{d}_G(q_0^i,q_1^j)^2]}+\sqrt{\sum_{j,k}\pi_{12}(j,k)[\mathcal{W}_2(\nu_1^{j},\nu_2^{k})^2+\gamma\tilde{d}_G(q_1^j,q_2^k)^2]}\\
    &=d_{V_2}(\rho_0,\rho_1)+d_{V_2}(\rho_1,\rho_2).
\end{align*}
The last inequality used the lemma above.
$\hfill\square$

\section{Geodesics of $d_{V_2}(\cdot,\cdot)$}\label{geodesic2}
\begin{theorem}
$$d_{V_2}(\rho_s,\rho_t)=(t-s)d_{V_2}(\rho_0,\rho_1), \qquad 0\le s<t\le1$$
\end{theorem}
\begin{proof}
  \begin{align*}
d_{V_2}(\rho_s,\rho_t)&\le\sqrt{\sum_{i,j}\pi^*(i,j)(\mathcal{W}_2(\nu_s^{i},\nu_t^{j})^2+\gamma\tilde{d}_G(q_s^i,q_t^j)^2)}\\
&=(t-s)\sqrt{\sum_{i,j}\pi^*(i,j)\mathcal{W}_2(\nu_0^{i},\nu_1^{j})^2+\gamma\tilde{d}_G(q_0^i,q_1^j)^2}\\
&=(t-s)d_{V_2}(\rho_0,\rho_1)
\end{align*}
By the triangle inequality,
\begin{align*}
d_{V_2}(\rho_0,\rho_1)&\le d_{V_2}(\rho_0,\rho_s)+d_{V_2}(\rho_s,\rho_t)+d_{V_2}(\rho_t,\rho_1)\\
&\le sd_{V_2}(\rho_0,\rho_1)+(t-s)d_{V_2}(\rho_0,\rho_1)+(1-t)d_{V_2}(\rho_0,\rho_1)\\
&=d_{V_2}(\rho_0,\rho_1).
\end{align*}
These two inequalities give the result $d_{V_2}(\mu_s,\mu_t)=(t-s)d_{V_2}(\mu_0,\mu_1)$.
\end{proof}

\end{appendices}
\bibliographystyle{plain}
\bibliography{refs}

\end{document}